\documentclass{article}

\usepackage{microtype}
\usepackage{graphicx}
\usepackage{subfigure}
\usepackage{booktabs} 

\usepackage[pagebackref=true]{hyperref}

\usepackage[accepted]{icml2024}

\usepackage{amsmath}
\usepackage{amssymb}
\usepackage{mathtools}
\usepackage{amsthm}
\usepackage{amsfonts}
\usepackage{thm-restate}

\usepackage[capitalize,noabbrev]{cleveref}

\usepackage{natbib}
\theoremstyle{plain}
\newtheorem{theorem}{Theorem}[section]

\newtheorem{lemma}[theorem]{Lemma}
\newtheorem{corollary}[theorem]{Corollary}
\theoremstyle{definition}
\newtheorem{definition}[theorem]{Definition}
\newtheorem{assumption}[theorem]{Assumption}
\theoremstyle{remark}
\newtheorem{remark}[theorem]{Remark}

\usepackage[textsize=tiny]{todonotes}

\usepackage{amsfonts}       
\usepackage{nicefrac}       
\usepackage{xcolor}         
\usepackage{centernot}
\usepackage{multirow}
\usepackage{multicol}
\usepackage{adjustbox}
\usepackage{caption}
\usepackage{float}
\usepackage{pgfkeys}
\usepackage{array}
\usepackage{tikz}
\usepackage[inline]{enumitem}
\usepackage{rotating}
\usepackage{pdflscape}
\usepackage{enumitem}
\usepackage{bbm}
\usepackage{mathrsfs}
\usepackage{wrapfig}
\usepackage{soul}
\usepackage{appendix}
\usepackage{etoolbox}
\usepackage{blkarray}

\usepackage{multirow}

\newcommand{\sd}[1]{{{\footnotesize±}{\scriptsize#1}}}

\newcommand{\Fc}{\mathcal F}
\newcommand{\Gc}{\mathcal G}

\newcommand{\Nb}{\mathbb{N}}
\newcommand{\Rb}{\mathbb{R}}
\newcommand{\Hr}{\mathrm{H}}

\newcommand{\agg}{\text{agg}}
\newcommand{\upd}{\text{upd}}
\newcommand{\comb}{\text{comb}}
\newcommand{\sub}{\operatorname{sub}}

\renewcommand{\hom}{\operatorname{hom}}
\newcommand{\Hom}{\operatorname{Hom}}

\newcommand{\gen}{\operatorname{gen}}

\newcommand{\readout}{\operatorname{readout}}
\newcommand{\tw}{\operatorname{tw}}
\newcommand{\td}{\operatorname{td}}
\newcommand{\spasm}{\operatorname{spasm}}
\newcommand{\Lip}{\operatorname{Lip}}

\usetikzlibrary{shapes.geometric}
\tikzset{gon/.style={name=tmp,regular polygon,regular polygon sides=#1,minimum
size=10pt,inner sep=0pt},flag connection/.style={-},
polygon side/.style args={#1--#2}{
insert path={(tmp.corner #1) edge[flag connection] (tmp.corner #2)}}}
\newcommand{\FlagGraph}[3][false]{
\ifnum#2=1%
\tikz[baseline=(tmp1)]{\node[circle,inner sep=1pt,minimum size=0.8mm,draw=black, fill=\ifbool{#1}{blue}{white}] (tmp1) at (0,0){};}
\else%
\ifnum#2=2%
 \tikz[baseline=(tmp1)]{
 \node[circle,inner sep=1pt,minimum size=0.8mm,draw=black, fill=\ifbool{#1}{blue}{white}] (tmp1) at (0,0){};
 \node[circle,inner sep=1pt,draw=black] (tmp2) at (10pt,0){};
 \ifx#3\empty%
 \else
 \draw (tmp1) edge[flag connection] (tmp2);
  \fi}
\else
\tikz[baseline=(tmp.south)]{\node[gon=#2,inner sep=0pt,minimum size=5mm]{};
\draw[draw=black, fill=\ifbool{#1}{blue}{white}] (tmp.corner 1) circle (1.5pt);
\foreach \X in {2,...,#2}{\draw (tmp.corner \X) circle (1.5pt);}
\draw[polygon side/.list={#3}]}
\fi
\fi
}

\icmltitlerunning{Generalization of Graph Neural Networks through the Lens of Homomorphism}
  
\begin{document}

\twocolumn[
\icmltitle{Generalization of Graph Neural Networks through the Lens of Homomorphism}



\icmlsetsymbol{equal}{*}

\begin{icmlauthorlist}
\icmlauthor{Shouheng Li}{anu,d61}
\icmlauthor{Dongwoo Kim}{postech}
\icmlauthor{Qing Wang}{anu}
\end{icmlauthorlist}

\icmlaffiliation{anu}{School of Computing, The Australian National University, Canberra, Australia}
\icmlaffiliation{d61}{Data61, CSIRO, Canberra, Australia}
\icmlaffiliation{postech}{CSE \& GSAI, POSTECH, Pohang, South Korea}

\icmlcorrespondingauthor{Shouheng Li}{shouheng.li@anu.edu.au}

\icmlkeywords{Machine Learning, ICML}

\vskip 0.3in
]



\printAffiliationsAndNotice 

\begin{abstract}
Despite the celebrated popularity of Graph Neural Networks (GNNs) across numerous applications, the ability of GNNs to generalize remains less explored. 
In this work, we propose to study the generalization of GNNs through a novel perspective — analyzing \emph{the entropy of graph homomorphism}. By linking graph homomorphism with information-theoretic measures, we derive generalization bounds for both graph and node classifications. These bounds are capable of capturing subtleties inherent in various graph structures, including but not limited to paths, cycles and cliques.
This enables a data-dependent generalization analysis with robust theoretical guarantees. To shed light on the generality of of our proposed bounds, we present a unifying framework that can characterize a broad spectrum of GNN models through the lens of graph homomorphism.
We validate the practical applicability of our theoretical findings by showing the alignment between the proposed bounds and the empirically observed generalization gaps over both real-world and synthetic datasets.
\end{abstract}
\section{Introduction}
\label{sec:intro}
Generalization is a fundamental area of machine learning research. 
Although understanding the generalization of neural networks remains a major challenge, recent findings~\citep{zhang2017rethink} show that neural networks demonstrate good generalization ability in the deep, over-parameterized regime. However, this observation does not apply to Graph Neural Networks (GNNs). As shown by \citet{CongRM21}, the generalization ability of GNNs tends to deteriorate with deeper architectures. This hints that the understanding of GNN generalization might require perspectives different from standard neural networks.

The success of GNNs in various learning tasks on network data
has drawn growing attention to the study of GNN generalization.
The correlation between observed generalization gap and complexity promotes analysis using classical complexity measures from statistical learning theory, such as Rademacher complexity, Vapnik–Chervonenkis (VC) dimension, and PAC-Bayes~\citep{esser2021sometimes,garg2020radmacher,ScarselliTH18, liao21pacbayes}. Nevertheless, these generalization bounds are mostly vacuous and rely on classical graph parameters such as the maximum degree that fails to fully capture the complex and intricate structures of graphs. The recent work of ~\citet{morris23meet} derives another VC-dimension bound that relates to the 1-dimensional Weisfeiler–Leman algorithm (1-WL)~\cite{weisfeiler1968reduction}. Although their analysis is insightful, the bound itself remains vacuous for most real-world applications.

In this work, we analyse GNN generalization from a novel angle - \emph{the entropy of graph homomorphism}. Through this, we establish a connection with the generalization measure by \citet{Chuang2021-ik} which arises from optimal transport costs between training subsets. We demonstrate that the entropy of graph homomorphism serves as a good indicator for GNN generalization. Built on this insight, we propose non-vacuous generalization bounds. 
In contrast to existing GNN generalization bounds that are often limited to a small class of (and sometimes simplified) models, the proposed generalization bounds are widely applicable to the following GNNs:
\begin{itemize}
\item \textbf{1-WL GNN and k-WL GNN}, e.g. \cite{xu2018powerful, KipfW17,morris19wl, maron2019provably}
\item \textbf{Homomorphism-Injected GNN (HI-GNN)} \cite{Barcelo2021-rs}
\item \textbf{Subgraph-Injected GNN (SI-GNN)}  \cite{Bouritsas2023-hj, zhao22stars, zhang21nested, bevilacqua22equivariant, Zhang_2023-ys}. 
\end{itemize}

Our contributions are:
\begin{itemize}
    \item By establishing a connection between graph homomorphism and information-theoretic measures, we propose widely-applicable and non-vacuous GNN generalization bounds that capture complex graph structures, for both graph and node classification.
    \item We empirically verify that the proposed bounds are able to characterise generalization errors on both real-world benchmark and synthetic datasets.
\end{itemize}

\section{Related Work}
\paragraph{Generalization Bounds}
From the findings on over-parameterized deep learning~\cite{zhang2017rethink}, it has been argued that classical model complexity-based notions such as VC-dimensions~\cite{blumer1989learnability} and Rademacher complexity~\cite{bartlett2002rademacher} are lacking explanatory power in understanding generalization of deep neural networks. This motivates works to propose new generalization bounds from the perspective of algorithm stability and robustness~\cite{Kawaguchi2022robustness, Dziugaite2018dp}, information theory~\cite{rodriguez2021wasserstein, Sefidgaran2022distortion, arora2018compression}, fractal dimensions~\cite{Camuto2021fractal, Dupuis2023fractal} and loss landscape~\cite{Chiang23loss}. Our work is closely related to the Wasserstein distance-based margin-based generalization bound proposed by \citet{Chuang2021-ik}.

\paragraph{Generalization in GNNs}
We first discuss GNN generalization studies in the node classification setting.
\citet{verma2019stability} derive generalization bounds of a single-layer GCN based on algorithmic stability.
\citet{zhang2020fast} also focus on the simplified single-layer design and analyze GNN generalization using tensor initialization and accelerated gradient descent.
\citet{zhou2021generalization} extends the algorithmic stability analysis from single-layer to general multi-layer GCN and shows that the generalization gap tends to enlarge with deeper layers.
\citet{cong2021provable} observe the same trends on deeper GNNs and propose the detach weight matrices from feature propagation in order to improve GNN generalization.
\citet{oono2020optimization,esser2021sometimes} analyze the problem from the angle of the classical Rademacher complexity.
\citet{tang2023towards} establish a bound in terms of node degree, training iteration, Lipschitz constant, etc. \citet{Li2022sampling} investigate GNNs that have a topology sampling strategy and characterize conditions where sampling improves generalization.

Several works study GNN generalization in the graph classification setting.
\citet{garg2020radmacher} propose a Rademacher complexity-based bound that is tighter than the VC-bound by \citet{ScarselliTH18}.
\citet{liao21pacbayes} develop a PAC-Bayesian bound that depends on the maximum node degree and the spectral norm of the weights.
\citet{ju2023diffusion} present improved generalization bounds for GNNs that scale with the largest singular value of the diffusion matrix. \citet{maskey2022random} assume that graphs are drawn from a random graph model and show that GNNs generalize better when trained on larger graphs. \citet{morris23meet} link VC-dimensions to 1-WL algorithm and bound it by the number of distinguishable graphs.

\paragraph{Homomorphism and Subgraph GNNs}
\citet{Nguyen2020hom} first explore the use of graph homomorphism counts in GNNs, showing their universality in approximating invariant functions. \citet{Barcelo2021-rs} suggest to combine homomorphism counts with a GNN.
On a slightly different route, \citet{bevilacqua22equivariant} represent graphs as a collection of subgraphs from a predetermined policy. \citet{zhao22stars} and \citet{zhang21nested} extend this idea by representing graphs with a set of induced subgraphs. \citet{Bouritsas2023-hj} use isomorphism counts of small subgraph patterns to represent a graph. In a similar spirit, \citet{thiede21auto} applied convolutions on automorphism groups of subgraph patterns. Instead of directly using subgraph counts, \citet{wijesinghe022, Wang_undated-iz} propose to inject local structure information into neighbour aggregation. The latter further shows the model expressivity grows with the size of subgraph patterns and the radius of aggregation. These works are also related to graph kernel methods that use subgraph patterns~\cite{Shervashidze2011-lw, horvath04, costa10}
\section{Preliminaries}

\subsection{Graph Homomorphism and Entropy}
We consider undirected and unlabelled graphs. Given two graphs $F = (V_F, E_F)$ and $G = (V_G, E_G)$, a \emph{homomorphism} $\varphi$ is a mapping $\varphi: V_F \rightarrow V_G$ such that $\{ \varphi(u),\varphi(v) \} \in E_G$ for all $\{u,v\} \in E_F$. 
The set of all homomorphisms from $F$ to $G$ is denoted as $ \Phi_{F\rightarrow G}$. 
We denote by $\hom(F,G)$ the number of homomorphisms from $F$ to $G$. A graph $G$ is rooted when a node $v\in V_G$ is declared as the root, the corresponding rooted graph is denoted as $G^v$. For rooted graphs $G^v$ and $F^w$, an homomorphism additionally maps $v$ to $w$.

Let $X_{F}: \Gc \to \Nb$ be a random variable such that $X_{F}(G)$ represents the homomorphism count $\hom(F,G)$. 
Given a set of graph patterns $\Fc = \{F_, \dots, F_{|\Fc|}\}$,
$X_{\Fc}:=\left(X_{F_1}, \dots, X_{F_{|\Fc|}}\right)$ is a multivariate random variable.
The rooted counterparts are denoted as $X_{F^r}$ and $X_{\Fc^r}$ respectively. 
The \emph{entropy} of $X_{F}$ is $\Hr(X_{F}):=-\sum_{r\in\Nb}\Pr(X_F(G)=r)\log(\Pr(X_F(G)=r))$. The entropy of $X_{\Fc}$, denoted $\Hr(X_{\Fc})$, is defined as the \emph{joint entropy} 
$\Hr(X_{F_1}, \dots, X_{F_{|\Fc|}})$. 
We can further extend the definition of joint entropy to a graph distribution $\mu$.

\begin{definition}[Entropy of Homomorphism]
Given a distribution of graphs $G\sim \mu$ and a set of graph patterns $\Fc = \{F_, \dots, F_{|\Fc|}\}$, \emph{entropy of homomorphism counts} of $F\in \Fc$ over $\mu$ is $\Hr(X_{\mu,F})$.
The entropy of homomorphism counts of $\Fc$ over $\mu$ is
\begin{equation*}
\Hr(X_{\mu,\mathcal{F}}) = \Hr(X_{\mu, F_1}, \dots, X_{\mu, F_{|\Fc|}}). 
\end{equation*}
\end{definition}

\subsection{Homomorphic Image and Spasm}\label{sec:hom_image}
Given two graphs $F$ and $F'$, if there exists a homomorphism $\varphi$ from $F$ to $F'$, then $F'$ is a \emph{homomorphism image} of $F$. If $\varphi$
is surjective, we call $F'$ a \emph{homomorphic image} of $F$. In other words, $F'$ is a simple graph that can be obtained from $F$ by possibly merging zero or more non-adjacent vertices~\cite{Curticapean2017-xp}. For instance, \FlagGraph{2}{1--2} is a homomorphic image of \FlagGraph{3}{1--2,2--3}. \citet{Curticapean2017-xp} proposed the concept of \emph{spasm} as the set of all homomorphic images of a given graph. For example, the spasm of \FlagGraph{4}{1--2,2--3,3--4,1--4}, denoted as $\spasm(\FlagGraph{4}{1--2,2--3,3--4,1--4})$, is 
\begin{equation*}
    \spasm(\FlagGraph{4}{1--2,2--3,3--4,1--4}) = \{
        \FlagGraph{4}{1--2,2--3,3--4,1--4}, \FlagGraph{3}{1--2,2--3}, \FlagGraph{2}{1--2}
    \}.
\end{equation*}
\citet{Curticapean2017-xp} further uses the multiset of homomorphism counts $\{\!\!\{\hom(F', G) | F' \in \spasm(F)\}\!\!\}$ as a basis to obtain the subgraph counts of $F$ in $G$. 

\subsection{$\mathcal{F}$-pattern Trees}
\citet{Dvorak2010-rv} and \citet{Dell2018-bg} explore an connection between tree-based substructures and $k$-WL test~\cite{cai1992optimal}.  \citet{Barcelo2021-rs} extends it to the concept of $\mathcal{F}$-pattern trees. 
Let $G^v$ denote the rooted graph $G$ whose root is vertex $v$.
A \emph{join graph} $(F_1\star F_2)^v$ is formed by combining $F_1^v$ and $F_2^u$ in a disjoint union by merging $u$ into $v$ and making $v$ the new root. As an example, the join graph of \FlagGraph[true]{3}{1--2,2--3,1--3} and \FlagGraph[true]{3}{1--2,1--3} is \FlagGraph[true]{5}{1--2,2--3,1--3,1--4,1--5}, where coloured nodes represent roots. Let $\mathcal{F}=\{F_1^r, \dots, F_{|\mathcal{F}|}^r\}$ be a set of rooted graphs, and $T^{r}=(V,E)$ be a tree with root $r$. An $\mathcal{F}$-pattern tree is obtained from the rooted tree $T^r$ followed by joining every vertex $w\in V$ with any number of copies of patterns from $\mathcal{F}$. The tree $T^{r}$ is called the \emph{backbone} of the $\mathcal{F}$-pattern tree. The \emph{depth} of an $\mathcal{F}$-pattern tree is the depth of its backbone. 
For example, given the backbone \FlagGraph[true]{3}{1--2,1--3}, the corresponding $\{\FlagGraph[true]{3}{1--2,2--3,1--3}\}$-pattern trees includes \FlagGraph[true]{3}{1--2,1--3}, \FlagGraph[true]{5}{1--2,1--3,1--4,1--5,4--5}, \FlagGraph[true]{5}{1--2,1--3,2--4,2--5,4--5} and \FlagGraph[true]{7}{1--2,1--3,2--4,2--5,4--5,3--6,3--7,6--7}, where in the first and the last graphs, zero and two copies of the pattern \FlagGraph[true]{3}{1--2,2--3,1--3} are joined to the backbone, respectively, and in the second and third, the pattern is joined to the root and non-root of the backbone, respectively.
Any tree can be used as the backbone to construct a $\mathcal{F}$-pattern tree.
We refer readers to \citet{Barcelo2021-rs} for detailed examples.
Patterns in $\mathcal{F}$ can be either rooted or unrooted. In practice, for symmetric patterns such as cliques or cycles, the choice of root node is irrelevant. 

\subsection{Graph Neural Networks}
\label{subsec:gnn}

Given a graph $G = (V, E)$, where $V$ is the node set and $E$ is the edge set, most GNNs take the following basic form:
\begin{equation}
\label{eqn:gnn_layer}
  h_v^{(l+1)} = f_{\upd}\left(h_v^{(l)},
                        {f_{\agg}}\left(\{\!\!\{h_u^{(l)}| u\in N(v)\}\!\!\}\right)
  \right)
\end{equation}
where $l$ is the GNN layer number, $h_v$ is the vector representation of a node $v\in V$. $N(v)$ is a set of neighbours of $v$, which can be defined differently depending on specific GNNs. $f_{\agg}$ is an aggregation function summarizing the representations of neighbours, and $f_{\upd}$ is an update function that combines the aggregated representation with the representation of node $v$ in the previous layer. The representation power of GNNs is often characterized by how one defines the aggregation and update functions along with the neighbourhood nodes.
When graph-level representation is required, an additional function $f_{\comb}$ is often used to combine representations of all nodes in a graph, i.e. 
\begin{equation}
\label{eqn:combine_function}
h_G^{(l+1)}=f_{\comb}(\{\!\!\{h_v^{(l+1)} | v \in V \}\!\!\}),
\end{equation}
where $h_G$ is the vector representation of a graph $G$.
When $f_{\upd}$, $f_{\agg}$ and $f_{\comb}$ are injective and $N(v)$ is the set of direct neighbours of $v$, as shown by \citet{xu2018powerful} and \citet{morris19wl}, GNNs of this form are bounded by the 1-dimensional Weisfeiler-Lehman test (1-WL) in distinguishing non-isomorphic graphs. 

Higher-order GNNs have been proposed to increase the expressivity beyond 1-WL, following the algorithmic design of k-WL. 
k-WL adopts a similar iterative refinement process as 1-WL, but instead of updating colours on nodes, k-WL updates colours on k-tuples of nodes ($k> 2$). 
While some of these GNNs are as strong as $k$-WL~\cite{maron2019provably} and some are weaker~\cite{morris2019weisfeiler}, they are all provably stronger than 1-WL.

Some GNNs integrate substructure information, in the form of either subgraphs~\cite{Bouritsas2023-hj, zhao22stars, zhang21nested, bevilacqua22equivariant, Zhang_2023-ys}, or homomorphism images~\cite{Barcelo2021-rs}, into a GNN layer. That is, \cref{eqn:gnn_layer} is adapted as
\begin{equation}
\label{eqn:gnn_layer_substructure}
  h_v^{(l+1)} =
  f_{\upd}\left(h^{(l)}_v,
    f_{\agg}\left(
        \{\!\!\{
            f_{\sub}(Z) | Z \in \mathcal{Z}
        \}\!\!\}
    \right)
  \right)
\end{equation}
where $\mathcal{Z}$ is a multiset of substructures, $f_{\sub}$ is a function that aggregates features from the substructure $Z$. 
Different from substructure injection, \citet{Welke_2023-xn} concatenate substructure counts into the final graph representations obtained from GNNs.
These GNNs can all go beyond 1-WL in terms of expressivity.

\subsection{$\Fc$-MPNN as a Unified GNN Framework}
\label{sec:gnn_hom}
Given a set of patterns $\mathcal{F}=\{F_1,\dots,F_{|\mathcal{F}|}\}$, we can stack their homomorphism numbers to $G$ to form a \emph{Lov{\'{a}}sz vector}~\cite{Welke_2023-xn} $\Hom(\mathcal{F}, G) = \left(\hom(F_1, G), \dots, \hom(F_{|\mathcal{F}|}, G)\right)$. Homomorphism numbers are isomorphism invariant, i.e., if two graphs $G_1$ and $G_2$ are isomorphic, then their Lov{\'{a}}sz vectors $\Hom(\mathcal{F}, G_1)$ and $\Hom(\mathcal{F}, G_2)$.
When $\mathcal{F}$ contains all graphs of size up to $|V_G|$, the Lov{\'{a}}sz vector is graph isomorphism complete, i.e., two graphs have the same Lov{\'{a}}sz vector if and only if they are isomorphic~\cite{lovasz2012}. \citet{Barcelo2021-rs} described a homomorphism-based GNN framework, named $\Fc$-MPNN, that unifies 1-WL GNN and their k-WL variants, as well as substructure-inject GNNs as descrbied in \cref{subsec:gnn}. $\Fc$-MPNN takes the following form:

\begin{definition}[$\Fc$-MPNN]
A $\Fc$-MPNN, parameterized by a set of patterns $\mathcal{F}$, iteratively updates the node representation $h_v$ of target node $v$ and the graph representation $h_G$ via
\begin{align}
 h^{(l+1)}_{\varphi}= & f_{\agg}\left(
    \{\!\!\{
        h^{(l)}_{\varphi(u)} | u \in V_{F}
    \}\!\!\}
 \right)  \\
 h^{(l+1)}_{F\rightarrow G^v} = & f_{\comb}\left(
    \{\!\!\{
        h^{(l+1)}_{\varphi} | \varphi \in  \Phi_{F\rightarrow G^v}
    \}\!\!\}
 \right) \\
 h^{(l+1)}_v = & f_{\upd}\left(
    h^{(l)}_v,
    [
        h^{(l+1)}_{F\rightarrow G^v} | F \in \mathcal{F}
    ]
  \right) \\
 h^{(l+1)}_G = & f_{\readout}\left(
    \left\{\!\!\left\{
        h^{(l+1)}_v | v \in V_G 
    \right\}\!\!\right\}
  \right) \label{eqn:hom_gnn_graph} 
\end{align}
where $\varphi$ is a homomorphism from $F$ to $G^v$,
$h_{\varphi}$ is the representation of the homomorphism image under $\varphi$, $h_{F\rightarrow G^v}$ is the aggregated representation of all homomorphism images of $F$ inside $G^v$. $h_{\varphi(u)}$ can be either a node feature or a node representation when used together with another GNN.
\end{definition}

It is easy to see that $h_G$ in \cref{eqn:hom_gnn_graph} resembles a Lov{\'{a}}sz vector when $f_{\agg}$ is an indicator function for non-empty multisets and $f_{\comb}$ is a sum function.

Let $T_L(\mathcal{F})$ and $T^r_L(\mathcal{F})$ denote the sets of unrooted and rooted $\mathcal{F}$-pattern trees of depth up to $L$, respectively.
From \citet{Barcelo2021-rs}, we know that two nodes $v$ and $w$ in a graph $G$ are indistinguishable by $\Fc$-MPNN if and only if $\hom(T^r, G^v) = \hom(T^r, G^w)$ for every $T^r\in T^r_L(\mathcal{F})$. Similarly, two graphs $G$ and $H$ are indistinguishable if and only if $\hom(T, G) = \hom(T, H)$ for every $T\in T_L(\mathcal{F})$.

We can compare the expressivity of $\Fc$-MPNNs by comparing their set relation of $\mathcal{F}$. That is, a model A with a pattern set $\mathcal{F}_A$ is more expressive than, or equally expressive to, a model B with a pattern set $\mathcal{F}_B$ if $\mathcal{F}_B\subset \mathcal{F}_A$. The expressivity gap between these two models can be quantified by $\mathcal{F}_A\setminus \mathcal{F}_B$.

\section{Generalization Bounds through Homomorphism}
\label{sec:bounds}

In this section, we propose generalization bounds for GNNs through homomorphism on graph classification. We then extend our analysis to node classification.

\paragraph{Analysis Setup.} 
For a classification problem of $K$ classes, let $\mathcal{X}$ be the input space of graphs (or nodes) and $\mathcal{Y}=\{1,\dots,K\}$ be the output space.
Following \citet{Chuang2021-ik}, we consider GNN as a composite hypythesis $f \circ \phi^{\mathcal{F},L}$ with a classifier $f\in \Psi$ and an encoder $\phi^{\mathcal{F},L}\in \Theta$. 
$f$ is a score-based classifier $f = [f_1, \dots, f_K]$ and $f_k \in \Psi_k$. 
$\phi^{\mathcal{F},L}$ is parameterized by the pattern set $\mathcal{F}$ and the number of layers $L$, and learns the graph representation (or node representation) $\phi^{\mathcal{F},L}(x)$. 
The prediction for $x\in\mathcal{X}$ is determined by $\arg\max_{y\in\mathcal{Y}} f(\phi^{\mathcal{F},L}(x))$.
Same as \citet{Chuang2021-ik, liao21pacbayes}, we assume that the multi-class $\gamma$-margin loss function is used. For a datapoint $(x,y)$, the margin of $f$ is defined by
\begin{equation}
\label{eqn: margin_loss}
    \rho_f(\phi^{\mathcal{F},L}(x), y):=f_y(\phi^{\mathcal{F},L}(x))-\max _{y^{\prime} \neq y} f_{y^{\prime}}(\phi^{\mathcal{F},L}(x))
\end{equation}
where $f$ misclassifies if $\rho_f(\phi^{\mathcal{F},L}(x), y)<0$. 
Let $\mu$ be a distribution over $\mathcal{X}\times \mathcal{Y}$. The dataset $S=\{x_i,y_i\}^m_{i=1}$ is drawn i.i.d from $\mu$.
Define $\mu_c$ as the marginal distribution over a class $c\in \mathcal{Y}$,
$m_c$ as the number of samples in class $c$, and
$\pi$ as the distribution of classes.
Denote the pushforward measure of $\mu_c$ w.r.t $\phi^{\mathcal{F},L}$ as $\phi_{\#}^{\mathcal{F},L}\mu_{c}$.

Let $R_\mu(f \circ \phi^{\mathcal{F},L})=\mathbb{E}_{(x, y) \sim \mu}\left[\mathbbm{1}_{\rho_f(\phi^{\mathcal{F},L}(x), y) \leq 0}\right]$ be the expected zero-one population loss and $\hat{R}_{\gamma, m}(f \circ \phi^{\mathcal{F},L})=\mathbb{E}_{(x, y) \sim S}\left[\mathbbm{1}_{\rho_f(\phi^{\mathcal{F},L}(x), y) \leq \gamma}\right]$ be the $\gamma$-margin empirical loss. We seek to bound the generalization gap $\gen(f\circ \phi^{\mathcal{F},L}) = R_\mu(f \circ \phi^{\mathcal{F},L}) - \hat{R}_{\gamma, m}(f \circ \phi^{\mathcal{F},L})$.

\subsection{Graph-level Generalization Bound by Homomorphism Entropy}

\begin{assumption}
\label{assumption:graph}
For graph classification, we assume that graphs are i.i.d samples.
\end{assumption}

Let
$\Omega(x)=\sqrt{\min(\frac{1}{2}x, 1-\exp(-x))}$, 
$\Delta(\cdot)$ be the diameter of a space,
$\beta_c= \Delta(\phi_{\#}^{\mathcal{F},L}\mu_{c})$,
$L_c = \operatorname{Lip}\left(\rho_f(\cdot, c)\right)$, and
$T_L(\mathcal{F})$ be the set of all $\mathcal{F}$-pattern trees up to depth $L$. We have the following corollary.
\begin{restatable}[Expectation Bound for Graph Classification]{corollary}{dataindependentgraphbound}
\label{lemma:non_class_bound}
Let  $\widetilde{D}_{KL}(\Fc, S, \tilde{S}) = D_{KL}\left( X_{\mu_S, T_L(\Fc)}\parallel  X_{\mu_{\tilde{S}}, T_L(\Fc)} \right)$.
Given $m$ i.i.d graph samples, with probability at least $1-\delta > 0$.
we have
\begin{align*}
   \gen&(f\circ \phi^{\Fc,L}) \leq \\
    & \mathbb{E}_{c \sim \pi}
    \left[
    \frac{L_c}{\gamma} 
        \mathbb{E}_{S,\tilde{S}\sim \mu_{c}^{m_c}}
        \left[ \beta_c\Omega\left(
        \widetilde{D}_{KL}(\Fc, S, \tilde{S})
        \right) \right]
    \right]\\
    & +\sqrt{\frac{\log (1 / \delta)}{2 m}}
\end{align*}
\end{restatable}

While the bound in \cref{lemma:non_class_bound} is useful for theoretical and data-independent comparison between GNNs, the expectation term over $S,\tilde{S}\sim \mu_c^{m_c}$ is intractable in general. To address this drawback, we derive another bound in \cref{lemma:graph_class_bound}, which can be computed via sampling.

\begin{restatable}[Data-dependent Bound for Graph Classification]{lemma}{datadependentgraphbound}
\label{lemma:graph_class_bound}
Given \cref{assumption:graph} and $K$ classes,
let $\{S^j, \tilde{S}^j\}^n_{j=1}$ be $n$ pairs of samples where each $S^j, \tilde{S}^j \sim \mu^{\lfloor m_c/2n \rfloor}_{c}$, and $\mu_{S_j}$ and $\mu_{\tilde{S}^j}$ be the corresponding empirical distributions, respectively. Also, let $\widehat{D}_{KL}(\Fc, S^j, \tilde{S}^j) = \frac{1}{n}\sum_{j=1}^n\left(\beta_c \Omega\left(D_{KL}( X_{\mu_{S^j}, T_L(\Fc)} \parallel X_{\mu_{\tilde{S}^j}, T_L(\Fc)}\right) \right)$ and
$m=\sum_{c=1}^K\lfloor\frac{m_c}{2n}\rfloor$. With probability at least $1-\delta > 0$, we have 
    \begin{align*}
        \gen&(f\circ \phi^{\Fc,L}) \leq \\
             & \mathbb{E}_{c \sim \pi}
        \left[
        \frac{L_c}{\gamma} 
        \left(
            \widehat{D}_{KL}(\Fc, S^j, \tilde{S}^j) + 2\beta_c\sqrt{\frac{\log(\frac{2K}{\delta})}{n\lfloor\frac{m_c}{2n}\rfloor}}
        \right)
        \right] \\
        & +\sqrt{\frac{\log (2 / \delta)}{2 m}}.
    \end{align*}
\end{restatable}

\subsection{Extending to Node-level Generalization Bound}
The above generalization bounds can be extended to node-level, subjecting to an additional strong assumption below. We define \emph{$L$-hop ego-graph} of a node $v$ as the induced subgraph formed by nodes within $L$ hops from $v$.

\begin{assumption}
\label{assumption:node}
Following \citet{Wu2022HandlingDS,garg2020radmacher,verma2019stability}, we assume that each node and its $L$-hop ego-graph are i.i.d.
\end{assumption}

With a slight abuse of notation, let $\nu_c$ be the marginal distribution of nodes and labels on $\mathcal{X}\times\mathcal{Y}$ in the class $c$, 
$\phi_{\#}^{\mathcal{F},L}\nu_{c}$ be the pushforward measure of $\nu_c$ w.r.t $\phi^{\mathcal{F},L}$, 
$\alpha_c= \Delta(X_{\nu_c,T_L(\Fc^r)})$,
$L_c = \operatorname{Lip}\left(\rho_f(\cdot, c)\right)$, and
$N_L(S)$ be the set of $L$-hop ego-graph of each node in $S$.
\begin{restatable}[Expectation Bound for Node Classification]{corollary}{dataindependentnodebound}
\label{lemma:node_no_bound}
Let  $\widetilde{D}_{KL}(\Fc^r, S, \tilde{S}) = D_{KL}\left( X_{\nu_S, T_L(\Fc^r)}\parallel  X_{\nu_{\tilde{S}}, T_L(\Fc^r)} \right)$.
Given $m$ samples and \cref{assumption:node},
with probability at least $1-\delta > 0$,
we have
\begin{align*}
   \gen&(f\circ \phi^{\mathcal{F},L}) \leq \\
    & \mathbb{E}_{c \sim \pi}
    \left[
    \frac{L_c}{\gamma} 
        \mathbb{E}_{S,\tilde{S}\sim \nu_{c}^{m_c}}
        \left[ \alpha_c\Omega\left(
        \widetilde{D}_{KL}(\Fc^r, S, \tilde{S})
        \right) \right]
    \right] \\
    & +\sqrt{\frac{\log (2 / \delta)}{2 m}}.
\end{align*}
\end{restatable}

\begin{restatable}[Data-dependent Bound for Node Classification]{lemma}{datadependentnodebound}
\label{lemma:node_class_bound}
Given \cref{assumption:node} and $K$ classes, let
$\{S^j, \tilde{S}^j\}^n_{j=1}$ be $n$ samples where $S^j, \tilde{S}^j \sim \nu^{\lfloor m_c/2n \rfloor}_{c}$  and $\hat{\nu}_{S_j}$ and $\hat{\nu}_{\tilde{S}^j}$ be the corresponding empirical distributions, respectively. Also, let 
$\widehat{D}_{KL}(\Fc^r, S^j, \tilde{S}^j) = \frac{1}{n}\sum_{j=1}^n\left(\alpha_c \Omega\left(D_{KL}( X_{\hat{\nu}_{S^j}, T_L(\Fc^r)} \parallel X_{\hat{\nu}_{\tilde{S}^j}, T_L(\Fc^r)}\right) \right)$.
Then with probability at least $1-\delta > 0$, we have
\begin{align*}
    \gen&(f\circ \phi^{\mathcal{F},L}) \leq \\
         &\mathbb{E}_{c \sim \pi}
    \left[
    \frac{L_c}{\gamma} 
    \left(
        \widehat{D}_{KL}(\Fc^r, S^j, \tilde{S}^j) + 2\alpha_c\sqrt{\frac{\log(\frac{2K}{\delta})}{n\lfloor\frac{m_c}{2n}\rfloor}}
    \right)
    \right] \\
    & +\sqrt{\frac{\log (2 / \delta)}{2 m}}.
\end{align*}
\end{restatable}



\section{Implications}
\label{sec:implication}

Given $\Fc$-MPNN described in ~\cref{sec:gnn_hom}, we seek to answer the question: \emph{how is the generalization bound affected by the choice of $\mathcal{F}$?} In this section, we answer the question based on \cref{lemma:non_class_bound}.
We first compare the entropy $\widetilde{D}_{KL}(\Fc, S, \tilde{S})$ on different choices of $\mathcal{F}$. We then discuss how $\beta_c$ and $L$ affect generalization bounds. Finally, we compare the bounds of GNNs described in \cref{sec:gnn_hom}.
For brevity, this section focuses on graph-level bounds, but similar results apply to node-level bounds.

\subsection{Implication of $\widetilde{D}_{KL}(\Fc, S, \tilde{S})$} 
\label{sec:implication_entropy}
In the following, we show two cases where $\widetilde{D}_{KL}(\Fc, S, \tilde{S})$ can be directly compared given different choices of $\mathcal{F}$.

\paragraph{Case 1: Gluing product of patterns.}

We start with the gluing product (disconnected union) of two patterns, e.g., the gluing product of $\FlagGraph{3}{1--2,2--3}$ and $\FlagGraph{2}{1--2}$ is $\FlagGraph{5}{1--2,2--3,4--5}$. 
If $F_1$ and $F_2$ are in $\mathcal{F}$, additionally adding the gluing product
of $F_1$ and $F_2$, denoted $F_1F_2$, to $\mathcal{F}$ does not increase $\mathrm{H}(X_{T_L(\mathcal{F})\rightarrow S})$. This is because $\hom(F_1F_2, S)=\hom(F_1,S)\hom(F_2,S)$~\cite{lovasz2012}. For rooted patterns, similar results can be obtained for the joining operation as described by \citet{Barcelo2021-rs}.

\paragraph{Case 2: Spasm of patterns}
If $F_1 \in \spasm(F_2)$, then $F_1$ can be constructed from $F_2$ by contracting zero or more non-adjacent nodes, i.e. $\hom(F_2,F_1) > 0$. Then it is easy to see $\Phi_{F_1\rightarrow S} \subseteq \Phi_{F_2\rightarrow S}$. 
The same can be derived for the corresponding $\mathcal{F}$-pattern trees.
Hence, if $F_1 \in \spasm(F_2)$, then 
$\widetilde{D}_{KL}(T_L(\{F_1\}), S, \tilde{S}) \leq \widetilde{D}_{KL}(T_L(\{F_2\}), S, \tilde{S})$.

In general, given two graphs $F_1$ and $F_@$, while the corresponding KL divergence cannott be computed exactly due to the unknown distributions $\mu_c$ and $\pi$, we can roughly compare them using Shear's lemma.

\begin{lemma}
(Shearer's lemma~\cite{shearer1986}) 
 Let $\mathcal{Q}$ be a family of subsets of $[n]=\{1,\dots,n\}$ (possibly with repeats) such that each member of $[n]$ appears in at least $t$ times across $\mathcal{Q}$. For a random vector $(X_1, \dots , X_n)$,
 \begin{equation*}
 \mathrm{H}(X_1,\dots,X_n)\leq\frac{1}{t}\sum_{Q\in \mathcal{Q}}\mathrm{H}(X_Q)
 \end{equation*}
 where $X_Q = (X_j:j\in Q)$.
\end{lemma}

We hereby give an example of using Shearer's lemma. Suppose that the graph $C_3 = \FlagGraph{3}{1--2,2--3,1--3}$ has the vertex set $\{1,2,3\}$.  Let $\{\varphi(1), \varphi(2), \varphi(3)\}$ be a set of nodes that corresponds to the homomophism $\varphi$.
Let $P_2 = \FlagGraph{2}{1--2}$, using Shearer's Lemma, we have 
\begin{align*}
\mathrm{H}&(X_{C_3 \rightarrow S}) = \frac{1}{2}\mathrm{H}(X_{\varphi(1)\rightarrow S}, X_{\varphi(2)\rightarrow S}, X_{\varphi(3)\rightarrow S})
\\ &\leq  \frac{1}{2}(\mathrm{H}(X_{\varphi(1)\rightarrow S},X_{\varphi(2)\rightarrow S})+ \mathrm{H}(X_{\varphi(2)\rightarrow S},X_{\varphi(3)\rightarrow S})
\\ &\text{\;\;\;\;} + \mathrm{H}(X_{\varphi(2)\rightarrow S},X_{\varphi(3)\rightarrow S}))
\\ & =\frac{3}{2}\mathrm{H}(X_{P_2\rightarrow S})
\end{align*}
Let $C_4=\FlagGraph{4}{1--2,2--3,3--4,1--4}$ be another pattern. Similarly, we can obtain $\mathrm{H}(X_{C_4 \rightarrow S}) \leq 2\mathrm{H}(X_{P_2\rightarrow S})$. 
Thus, the entropy of using \FlagGraph{3}{1--2,2--3,1--3} as a pattern is likely to be lower than \FlagGraph{4}{1--2,2--3,3--4,1--4}.
So it is also likely that $\widetilde{D}_{KL}(C_3, S, \tilde{S}) \leq \widetilde{D}_{KL}(C_4, S, \tilde{S})$.
Note that $P_2$ is used as the base for comparison in this example, but other base patterns can also be used as long as they are in the common $\spasm$ set of the patterns to compare. 
For example, \FlagGraph{3}{1--2,2--3} can be used as the base to compare 
\FlagGraph{4}{1--2,2--3,3--4,1--4} and \FlagGraph{4}{1--2,2--3,3--4}. 

\subsection{Implication of $\beta_c$ and $L$}
$\Omega(\cdot)$ is upper bounded by 1. Thus, when the graphs in $S$ are structurally diverse, $\Omega\bigl(\widetilde{D}_{KL}(\Fc, S, \tilde{S})\bigr)$ can reach the upper bound. In this case, the diameter $\beta_c= \Delta(\phi_{\#}^{\mathcal{F},L}\mu_{c})$ plays a dominant role in the bound. Because the Wasserstein distance is defined in terms of Euclidean distance, $\beta_c$ is also measured by Euclidean distance. Hence, when $|\mathcal{F}|$ grows, $\beta_c$ is likely to increase, or at least remain the same. The latter case holds if some patterns in $\mathcal{F}$ have zero homomorphisms in $S$. The same result holds for the layers $L$, as a larger $L$ leads to a larger $T_L(\mathcal{F})$.

\begin{restatable}[]{corollary}{dataindependentcomparison}
\label{lemma:GNN_compare_non_data}
The following holds for the generalization bounds described in \cref{lemma:non_class_bound} and \cref{lemma:node_no_bound}.
\begin{enumerate}
    \item For a fixed $\mathcal{F}$, the bounds at $L+1$ is larger or equal to the bounds at $L$.
    \item Given $\mathcal{F}'_{\hom} \supset \mathcal{F}_{\hom}$ ($\mathcal{F}'_{\sub}\supset \mathcal{F}_{\sub}$, resp.), for a fixed $L$, 
    the bounds of the HI-GNN (SI-GNN resp.) with $\mathcal{F}'_{\hom}$ ($\mathcal{F}'_{\sub}$ resp.) is higher than the HI-GNN (SI-GNN resp.) with $\mathcal{F}_{\hom}$ ($\mathcal{F}_{\sub}$ resp.)
    \item Given a fixed $L$, if $\mathcal{F}_{\hom}\neq \varnothing$ ($\mathcal{F}_{\sub}\neq \varnothing$, resp.), the bounds for HI-GNN (SI-GNN resp.) is larger than or equal to the one for 1-WL GNN. The equality holds when $\mathcal{F}_{\hom} = \{\FlagGraph{1}{}\}$ ($\mathcal{F}_{\sub} = \{\FlagGraph{1}{}\}$, resp.).
    \item Given a fixed $L$, the bounds for HI-GNN (SI-GNN resp.) is smaller than $k$-WL GNNs where $k$ is the largest treewidth of a pattern in $\mathcal{F}_{\hom}$ ($\mathcal{F}_{\sub}$ resp.) and $k>2$.
\end{enumerate}
\end{restatable}


\begin{figure}[ht]
\centering     
\subfigure[ENZYMES]{\label{fig:hom_bound_enzymes}\includegraphics[width=0.49\columnwidth]{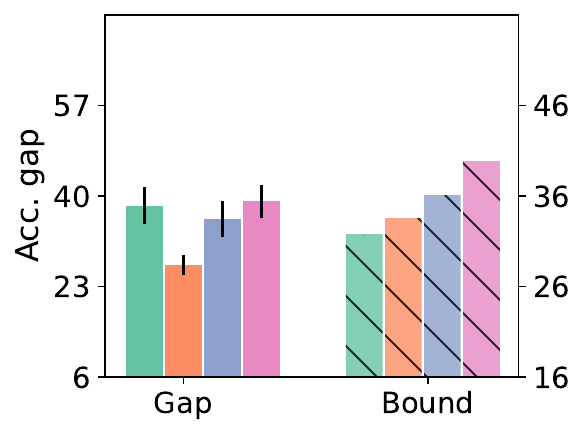}}
\subfigure[PROTEINS]{\label{fig:hom_bound_proteins}\includegraphics[width=0.49\columnwidth]{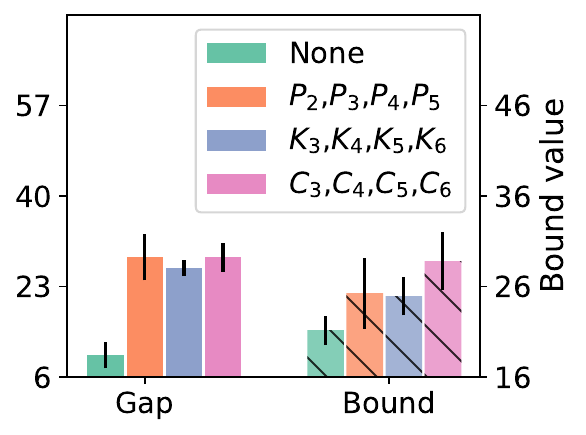}}
\caption{Accuracy gap and bound value using different homomorphism pattern sets, obtained from GCN of 4 layers.}
\label{fig:hom_bound}
\end{figure}

\section{Empirical Studies}
\paragraph{Compute  $X_{\hat{\mu}_{S^j}, T_L(\Fc)}$ and $X_{\hat{\mu}_{\tilde{S}^j}, T_L(\Fc)}$.}
An intuitive choice of $X_{\hat{\mu}_{S^j}, T_L(\Fc)}$ and $X_{\hat{\mu}_{\tilde{S}^j}, T_L(\Fc)}$ is the Lov{\'{a}}sz vector where each element in the vector is a homomorphism count of a pattern in $T_L(\mathcal{F})$. However, for GNNs listed in \cref{sec:gnn_hom}, the corresponding $T_L(\mathcal{F})$ is too large to compute the homomorphism counts explicitly:
for 1-WL GNNs, $T_L(\mathcal{F})$ contains all trees up to depth $L$, for homomorphism-injected GNNs, $T_L(\mathcal{F})$ contains all $\mathcal{F}$-pattern trees up to depth $L$~\citep{Barcelo2021-rs}. Fortunately, we can use colour histograms obtained from the $\mathcal{F}$-WL~\cite{Barcelo2021-rs} algorithm as an equally expressive choice for $X_{\hat{\mu}_{S^j}, T_L(\Fc)}$ and $X_{\hat{\mu}_{\tilde{S}^j}, T_L(\Fc)}$. 

A colour histogram is essentially a vector representation of a graph, where each element is the count of the appearance of a particular node colour.
The node-level $X_{\nu_S, T_L(\Fc^r)}$ and  $X_{\nu_{\tilde{S}}, T_L(\Fc^r)}$ can also be computed in this way, where a node is represented as a histogram of its neighbours' node colours. We stack node colour counts of each iteration to form the final node representation.
Details about colour histogram representation can be found in the well-known work of WL subtree graph kernel~\cite{Shervashidze2011-lw}.

\paragraph{Estimate KL Divergence.}
The KL divergence in \cref{lemma:graph_class_bound} and \cref{lemma:node_class_bound} measures the divergence of two sampled distributions.
Because the distributions are multi-dimensional, we estimate the corresponding multivariate KL divergence according to \citet{Perez-Cruz08}.

\subsection{Experiment Evaluation}

\begin{figure}
\centering     
\subfigure[PubMed: 4 layers]{\label{fig:hom_bound_pubmed}\includegraphics[width=0.49\columnwidth]{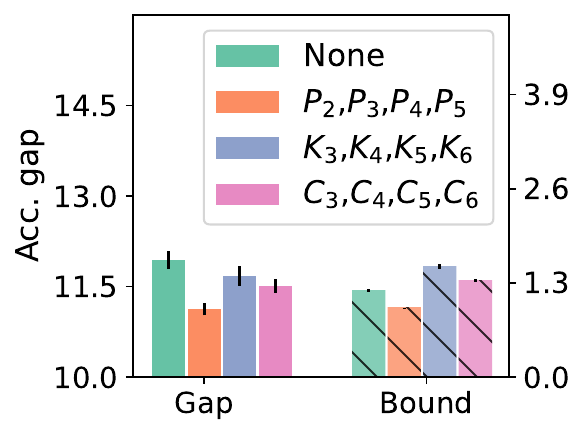}}
\subfigure[PATTERN: 16 layers]{\label{fig:hom_bound_pattern}\includegraphics[width=0.49\columnwidth]{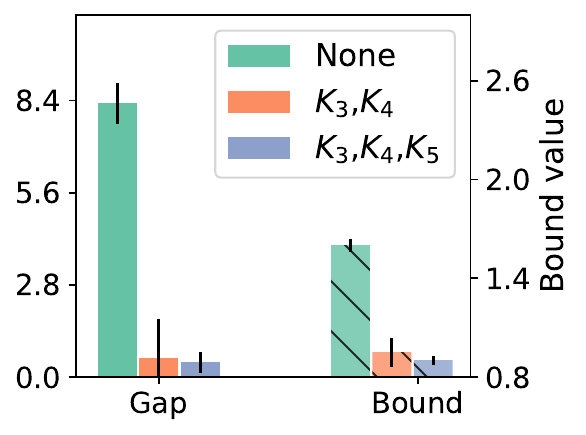}}
\caption{Accuracy gap and bound value using different homomorphism pattern sets}
\label{fig:hom_bound_node}
\end{figure}

We investigate how well the bounds align with observed generalization gap. In particular, we evaluate the generalization gaps of GCN on several datasets across node and graph classification tasks, with different numbers of layers and pattern sets, then compare the gaps with the proposed generalization bounds.

\paragraph{Tasks and Datasets}
We evaluate graph and node classification tasks. For graph classification, we use ENZYMES, PROTEINS and PTC-MR from the TU datasets~\cite{Morris+2020}. Each dataset is randomly split into 90\%/10\% for training and test. For node classification, we evaluated the generalization gap on Cora, CiteSeer, and PubMed citation datasets~\cite{SenNBGGE08,YangCS16}, each of which is divided into 60\% /40\% for training and test. We use the differences in accuracy and loss between training and test as indicators of the generalization gap. We train 2,000 and 500 iterations for node and graph classification tasks, respectively. Each training is repeated five times to obtain mean accuracy and standard deviation. For node classification, we additionally evaluate the inductive setting on the PATTERN dataset~\citep{DwivediJL0BB23}, which has over 12,000 artificially generated graphs resembling social networks or communities. The node classification task on PATTERN predicts whether a node belongs to a particular cluster or pattern. As shown by \citet{Barcelo2021-rs}, classification accuracy is greatly improved on PATTERN after injecting homomorphism information.

\paragraph{Setup and Configuration}
To calculate the bounds in \cref{lemma:graph_class_bound} and \cref{lemma:node_class_bound}, one needs to compute the ratio of Lipschitz constant and loss margin $\frac{L_c}{\gamma}$. As we focus on the upper-bound generalization gap, it is easier computationally to set the ratio to a ``safe" constant. In practice, we set $\frac{L_c}{\gamma}$ to 3 and 6 for graph and node tasks respectively. The two numbers are ``safe" because empirically they are larger than the ones estimated following the method used by \citet{Chuang2021-ik}. We use the largest pair-wise distance from the training set as the value of diameter in \cref{lemma:graph_class_bound} and \cref{lemma:node_class_bound}. We set $\delta$ to 0.01, so the computed bound is of high probability.

We use GCN~\citep{KipfW17} as the base model and inspect how the bounds capture the trend of changes in generalization gaps w.r.t.
\begin{enumerate*}[label=(\roman*)]
  \item number of layers, and
  \item ingested homomorphism and subgraph counts
\end{enumerate*}.

\subsection{Results and Discussion}

\begin{figure}
\centering     
\subfigure[ENZYMES]{\label{fig:hom_sub_bound_enzymes}\includegraphics[width=0.49\columnwidth]{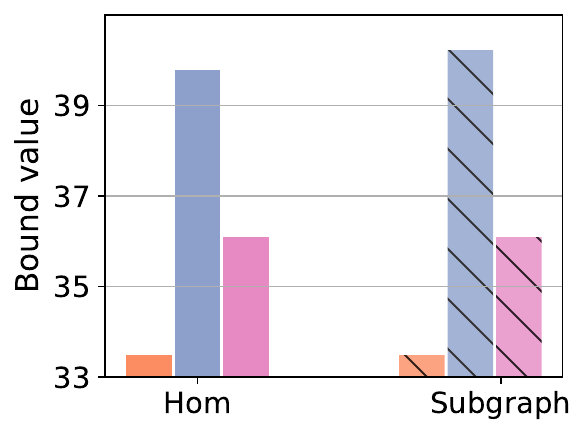}}
\subfigure[PubMed]{\label{fig:hom_sub_bound_pubmed}\includegraphics[width=0.49\columnwidth]{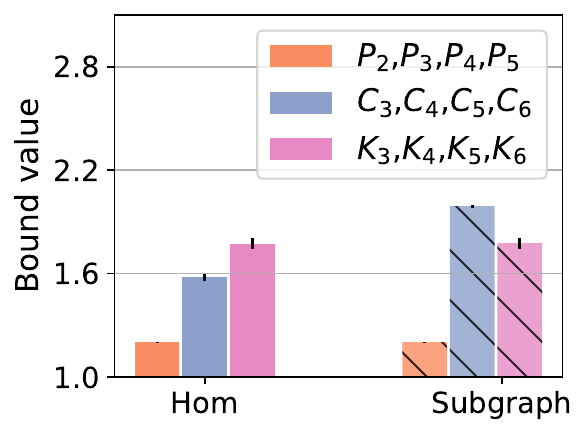}}
\caption{Bounds of homomorphism-injected and subgraph-injected models of 4 layers.}
\label{fig:hom_iso_bound}
\end{figure}

\begin{table}[t]
    \centering
    \resizebox{.9\columnwidth}{!}{
    \begin{tabular}{ccccc}
        \hline
         \multirow{2}{*}{Layer}{}&  & \multicolumn{3}{c}{Dataset} \\ \cmidrule{3-5}
         & & ENZYMES & PROTEINS & PTC-MR  \\
         \hline
         \multirow{4}{*}{1} & Acc. gap & 18.92\sd{4.28} & 7.22\sd{1.14} & 38.92\sd{5.54}\\
          & Loss gap & 0.69\sd{0.06} & 0.11\sd{0.00} & 1.03\sd{0.12} \\
          & Bound & 1.64\sd{0.0.00} & 1.51\sd{0.47} & 1.74\sd{0.00}\\
          & \# Histogram & 586 & 929 & 200 \\ \midrule
         \multirow{4}{*}{2} & Acc. gap & 20.78\sd{4.23} & 7.65\sd{3.77} & 40.24\sd{3.01}\\
          & Loss gap & 0.64\sd{0.08} & 0.10\sd{0.01}& 0.95\sd{0.09} \\
          & Bound & 10.02\sd{0.00} & 5.72\sd{0.71} & 3.70\sd{0.03}\\
          & \# Histogram & 595 & 996 & 257 \\ \midrule
         \multirow{4}{*}{3} & Acc. gap & 32.48\sd{1.57} & 9.33\sd{0.85} & 48.94\sd{5.88}\\
          & Loss gap & 0.89\sd{0.06} & 0.15\sd{0.15} & 1.25\sd{0.14} \\
          & Bound & 23.21\sd{0.00} & 14.82\sd{1.15} & 7.98\sd{0.08}\\
          & \# Histogram & 595 & 996 & 258 \\ \midrule
         \multirow{4}{*}{4} & Acc. gap & 38.18\sd{3.50} & 10.05\sd{2.43} & 47.13\sd{1.46} \\
          & Loss gap & 1.12\sd{0.07} & 0.21\sd{0.03} & 1.48\sd{0.25} \\
          & Bound & 31.73\sd{0.00} & 21.13\sd{1.63} & 13.43\sd{0.17}\\
          & \# Histogram & 595 & 996 & 258\\ \midrule
         \multirow{4}{*}{5} & Acc. gap & 43.22\sd{3.92} & 11.10\sd{1.23} & 44.15\sd{4.72}\\
          & Loss gap & 1.28\sd{0.21} & 0.28\sd{0.04} & 1.61\sd{0.20}\\
          & Bound & 36.10\sd{0.00} & 25.88\sd{2.23} & 18.26\sd{0.27}\\
          & \# Histogram & 595 & 996 & 258 \\ \midrule
         \multirow{4}{*}{6} & Acc. gap & 41.03\sd{3.07} & 12.30\sd{3.96} & 47.40\sd{3.67}\\
          & Loss gap & 1.38\sd{0.14} & 0.32\sd{0.07} & 1.60\sd{0.04} \\
          & Bound & 40.73\sd{0.00} & 29.13\sd{3.10} & 21.25\sd{0.08}\\
          & \# Histogram & 595  & 996 & 258 \\ \midrule
    \end{tabular}
    }
    \caption{Graph classification accuracy and loss gap using different numbers of layers, compared with bound and number of colour histograms (graphs distinguishable by 1-WL).}
    \label{tab:graph_layer_gap}
\end{table}

\paragraph{Bound w.r.t Number of Layers.}
As listed in \cref{tab:graph_layer_gap} and \cref{tab:node_layer_gap}, increasing layers increases both accuracy and loss gaps, as well as the generalization bound. We also listed the VC dimension bound proposed by \citet{morris23meet} in \cref{tab:graph_layer_gap}, which is essentially the number of graphs distinguishable by 1-WL (number of colour histograms). Compared with our bound, the VC bound cannot reflect generalization gap changes after the number of colour histograms stabilizes. For example, the number of colour histograms no longer changes after two layers for ENZYMES while the generalization gap continues to increase. Compared with the bounds in \citet{liao21pacbayes, garg2020radmacher} that are at least in the order of $10^4$, our bound is less vacuous as it matches the scale of the empirical gap.

\paragraph{Bound w.r.t Graph Patterns.}
Beyond the vanilla GCN, we are interested in understanding if the proposed bound can reflect changes in the generalization gap w.r.t. different graph patterns that are used to inject homomorphism and subgraph counts. The evaluated HI-GNN corresponds to LGP-GNN~\citep{Barcelo2021-rs}, the evaluated SI-GNN corresponds to a variant of GSN~\citep{Bouritsas2023-hj} without dividing subgraph counts by the number of automorphisms. In \cref{fig:hom_bound}, we plot the accuracy gap vs generalization bound from three different pattern sets, compared with the vanilla GCN. We use $P_n$, $K_n$ and $C_n$ to denote n-path, n-clique and n-cycle graphs, respectively, and use ``None" to denote vanilla GCN. We observe that different pattern sets cause different changes in the generalization gap. Namely, in ENZYMES, the path patterns lead to a smaller gap than cliques and cycle. Cycles lead to a larger gap than cliques. These changes in the generalization gap are also reflected in the corresponding bound, except for ENZYMES, where the gap for vanilla GCN is higher than others but the bound is lower. We also observe a similar trend in \cref{fig:hom_bound_node} for node classification. We found that in cases where injecting homomorphism information reduces generalization gaps, the corresponding $\mathcal{F}$-WL algorithm requires fewer iterations than 1-WL to stabilise. This contributes to smaller diameter and KL divergence and further leads to smaller bounds.

\paragraph{Homomorphism vs Subgraph Counts.}
Finally, we compare homomorphism-injected models with subgraph-injected models using the same pattern set, as shown in \cref{fig:hom_iso_bound}. We note for cycles, subgraph counts tend to cause a higher bound value than homomorphism counts, while the two bound values are the same for paths and cliques. Recall a pattern $F$'s subgraph count can be derived from homomorphism counts of the corresponding $\spasm(F)$. This observation matches the expectation that, in general, subgraph counts have a larger entropy than homomorphism counts, thus results in a larger bound value. For cliques, the $\spasm$ of a clique only contains itself, i.e. the homomorphism count of a clique equals its subgraph count, so the bound value remains the same. For paths, the $\spasm$ of each path in $\{P_2,P_3,P_4,P_5\}$ largely overlaps with the set itself $\{P_2,P_3,P_4,P_5\}$, leading to a smaller variance and thus smaller bound value. For cycles, the $\spasm$ of each cycle contains paths of smaller size, and these paths do not overlap with the cycle set, leading to larger variance and entropy and, thus, a larger bound value.

\begin{table}[t]
    \centering
    \resizebox{.83\columnwidth}{!}{
    \begin{tabular}{ccccc}
        \hline
         \multirow{2}{*}{Layer}{}&  & \multicolumn{3}{c}{Dataset} \\ \cmidrule{3-5}
         & & Cora & CiteSeer & PubMed\\
         \hline
         \multirow{3}{*}{1} & Acc. gap & 13.73\sd{0.04} & 25.51\sd{0.03}& 1.64\sd{0.01} \\
          & Loss gap & 0.48\sd{0.00} & 1.12\sd{0.01}& 0.03\sd{0.00}\\
          & Bound & 4.18\sd{0.00} & 4.63\sd{0.02} & 0.47\sd{0.01}\\
          \midrule
         \multirow{3}{*}{2} & Acc. gap & 13.41\sd{0.19} & 28.16\sd{0.21} & 4.46\sd{0.20}\\
          & Loss gap & 0.60\sd{0.01} & 1.84\sd{0.03} & 0.11\sd{0.00} \\
          & Bound & 4.57\sd{0.00} & 4.63\sd{0.02} & 0.55\sd{0.02}\\
          \midrule
         \multirow{3}{*}{3} & Acc. gap & 14.09\sd{0.49} & 29.72\sd{0.35} & 10.36\sd{0.10}\\
          & Loss gap & 0.98\sd{0.04} & 2.83\sd{0.12} & 0.34\sd{0.00} \\
          & Bound & 6.09\sd{0.13} & 5.09\sd{0.02} & 1.03\sd{0.05}\\
          \midrule
         \multirow{3}{*}{4} & Acc. gap & 14.39\sd{0.34} & 30.48\sd{0.13} & 11.93\sd{1.47} \\
          & Loss gap & 1.21\sd{0.08} & 3.79\sd{0.08} &  0.46\sd{0.02} \\
          & Bound &  6.97\sd{0.14}& 5.73\sd{0.03}& 1.45\sd{0.02} \\
          \midrule
         \multirow{3}{*}{5} & Acc. gap & 14.65\sd{0.55} & 30.77\sd{0.41} & 11.99\sd{0.12} \\
          & Loss gap & 1.62\sd{0.12} & 5.15\sd{0.22}& 0.52\sd{0.02}\\
          & Bound & 7.96\sd{0.15}& 6.48\sd{0.05} & 1.77\sd{0.02}\\
          \midrule
         \multirow{3}{*}{6} & Acc. gap & 14.96\sd{0.68}& 30.22\sd{0.41} & 11.90\sd{0.25} \\
          & Loss gap & 1.93\sd{0.11} & 6.52\sd{0.25} & 0.52\sd{0.02}\\
          & Bound & 7.96\sd{0.15}& 7.24\sd{0.06}& 1.77\sd{0.02}\\
          \midrule
    \end{tabular}
    }
    \caption{Node classification accuracy and loss gap using different numbers of layers, compared with bound value.}
\label{tab:node_layer_gap}
\end{table}

\section{Conclusion}
In this study, we delve into the generalization ability of GNNs by examining the entropy inherent in graph homomorphism, and integrating it with information-theoretic measures. This approach enables us to devise graph structure-aware generalization bounds that are useful for both graph and node classification tasks. These bounds, applicable to various GNNs within a unified homomorphism framework, closely match the actual generalization gaps observed in empirical evaluations on both real-world and synthetic datasets.


\bibliographystyle{icml2024}
\bibliography{references}  

\clearpage
\appendix
\section{Notations}
We list the notations used in \cref{tab:notations}.

\begin{table*}[ht]
    \centering
    \begin{tabular}{r|l}
    \hline
        $f_{\agg}$, $f_{\upd}$, $f_{\comb}$ $f_{\readout}$ & functions used in the aggregation, update, combine and readout steps of GNN \\
         $h_v$, $h_G$ & the vector representation of node $v$, graph $G$ resp. \\
         $h_{\varphi}$ & the vector representation of the homomorphism image under $\varphi$\\
         $h_{F\rightarrow G}$ & the aggregated representation of all homomorphism images from $F$ to $G$\\
         $F$, $G$ & graphs \\
         $V_G$, $E_G$ & the ndoe and edge sets of G \\
         $\varphi: V_F\rightarrow V_G$ & a homomorphism mapping from F to G\\
         $\Phi_{F\rightarrow G}$ & the set of all homomorphism mappings from $F$ to $G$\\
         $\mathcal{F}$ & a set of graph patterns \\
         $F_i$ & a unrooted pattern graph \\
         $F^r_i$ & a pattern graph rooted at node $r$\\
         $\hom(F,G)$ & the count of homomorphisms from $F$ to $G$\\
         $\Hom(\mathcal{F}, G)$ &  Lov{\'{a}}sz vector $\left(\hom(F_1, G), \dots, \hom(F_{|\mathcal{F}|}, G)\right)$\\
         $X_{F\rightarrow G}$& a random variables representing a homomorphism mapping $\varphi$ uniformly chosen from $\Phi_{F\rightarrow G}$.\\
         $X_{\mathcal{F}\rightarrow G}$& a collection of random variables $\left(X_{F_1\rightarrow G}, \dots, X_{F_{|\mathcal{F}|}\rightarrow G}\right)$\\
         $\mathrm{H}(X_{F\rightarrow G})$ & the entropy of the random variable $X_{F\rightarrow G}$\\
         $\mathrm{H}(X_{\mathcal{F}\rightarrow G})$ & the joint entropy $\mathrm{H}(X_{F_1\rightarrow G}, \dots, X_{F_{|\mathcal{F}|}\rightarrow G})$ \\
         $\spasm(G)$ & the set of all homomorphic images of G\\
         $\mathcal{W}_p(\mu, \nu)$ & $p$-Wasserstein distance of two probability distributions $\mu$ and $\nu$ \\
         $\operatorname{TV}(\mu,\nu)$ &  total variation between two distributions $\mu$ and $\nu$ \\
         $\Delta(\mathcal{X})$ & the diameter of the space $\mathcal{X}$ \\
        $D_\text{KL}(\mu\|\nu)$ & KL-divergence of distributions $\mu$ and $\nu$\\
        $\operatorname{Var}_{k,p}(\mu)$ & Wasserstein-$p$ $k$-variance of the distribution $\mu$\\
        $\operatorname{Var}_{k}(\mu)$ & Wasserstein-$1$ $k$-variance of the distribution $\mu$\\
        $\phi$, $\Theta$ & a feature learner and its space\\
        $f$, $\Psi$ & a classifier and its function space\\
        $\mathcal{Y}$ & output space of a classifier of $K$ classes\\
        $\mathcal{X}$ & vector input space of a feature learner\\
        $\rho_f(\cdot)$ & a multi-class margin loss function of the classifier $f$\\
        $\hat{R}_{\mu}(\cdot)$ & expected population loss\\
        $\hat{R}_{\gamma, m}(\cdot)$ & $\gamma$-margin empirical loss on $m$ samples\\
        $m$ & the number of samples \\
        $m_c$ & the number of samples in class $c$ \\
        $\mu_c$, $\nu_c$ & the marginal distribution of data in class $c$ \\
        $\phi_{\#}\mu$ & the push-forward measure of the distribution $\mu$ w.r.t the feature learner $\phi$ \\
        $p$ & the distribution of classes \\
        $\gen(f\circ \phi)$ & the generalization gap of a model compose of $f$ and $\phi$\\
        $\Lip(\cdot)$ & the margin Lipschitz constant of a function \\
        $\tw(F)$ & the tree width of $F$ \\
        $\td(F)$ & the depth of the tree graph $F$ \\
        $\Delta(\cdot)$ & diameter of a space \\
        $T_L(\mathcal{F})$ & the set of rooted $\mathcal{F}$-pattern trees of depth at most $L$ \\
        $\mathcal{F}_{\hom}$, $\mathcal{F}_{\sub}$ & the set of homomorphism/subgraph patterns used by HI-GNN/SI-GNN\\
    \hline
    \end{tabular}
    \caption{Notatoin Table}
    \label{tab:notations}
\end{table*}

\section{Additional Experiment Configuration}
We follow the same setup from \citet{morris23meet, tang2023towards, CongRM21} to remove regularizations including dropout and weight decay, and use the Adam optimizer. We use 64 as the inner dimension for intermediate layers. We have also run the experiments with 128 inner dimension and observed similar results.
We train for 2000 and 500 iterations for node and graph tasks, both using 0.001 as the learning rate. We record the best accuracy and loss for both training and test set and report the difference.

\section{Missing Proofs}

Before we prove the proposed bounds, we introduce a few definitions and theorems from previous works, as they are useful in the later proofs.

\subsection{Margin Bound with Wasserstein Distance}
\begin{restatable}[$p$-Wasserstein Distance~\cite{Chuang2021-ik}]{definition}{pwassersteindistance}
Let $\mu$ and $\nu\in\textsc{Prob}(\mathbb{R}^d)$ be two probability measures. The \emph{$p$-Wasserstein distance} with the Euclidean cost function $\parallel\cdot \parallel$ is defined as
\begin{equation*}
    \mathcal{W}_p(\mu, \nu):=\inf _{\pi \in \Pi(\mu, \nu)}\left(\mathbb{E}_{(x, y) \sim \pi}\parallel x - y \parallel^p\right)^{\frac{1}{p}}
\end{equation*}
where $\Pi(\mu, \nu)\subseteq \textsc{Prob}(\mathbb{R}^d \times \mathbb{R}^d)$ denotes the set of all couplings measure whose marginals are $\mu$ and $\nu$, respectively.
\end{restatable}

\begin{remark}
    Wasserstein distance measures the minimal cost to transport a distribution $\mu$ to a distribution $\nu$.
\end{remark}

\begin{restatable}[Wasserstein-$p$ $k$-variance~\cite{solomonGN22}]{definition}{kvardefinition}
Given a probability distribution $\mu\in \textsc{Prob}(\mathbb{R}^d)$ and a parameter $k\in \mathbb{N}$, the \emph{Wasserstein-$p$ $k$-variance} is defined as
\begin{equation*}
    \operatorname{Var}_{k,p}(\mu) = c_p(k, d) \cdot \mathbb{E}_{S,\tilde{S} \sim \mu^k} \left[ \mathcal{W}^p_p(\mu_S, \mu_{\tilde{S}}) \right],
\end{equation*}
where $\mu_S$ and $\mu_{\tilde{S}}$ are empirical measures of the same size and $c_p(k,d)$ is a normalization term.
\end{restatable}

\begin{remark}[Wasserstein-$1$ $k$-variance]
In this work, we focus on the Wasserstein-1 $k$-variance: 
\begin{equation*}
\operatorname{Var}_k(\mu) := \mathbb{E}_{S,\tilde{S}\sim \mu^k} [\mathcal{W}_1(\mu_S, \mu_{\tilde{S}})].
\end{equation*}
\end{remark}

\citet{Chuang2021-ik} studied generalization using Wasserstein distance and revealed that the concentration and separation of features play crucial roles in generalization. 

\begin{restatable}[\citet{Chuang2021-ik}]{lemma}{chuangbound}
\label{lemma:chuang_bound}
    Let $f = [f_1,\dots, f_K] \in \Psi$  where $\Psi = \Psi_1 \times \dots \times \Psi_K$, $f_i\in \Psi_i$, and $\Psi_i: \mathcal{X} \rightarrow \mathbb{R}$.
    Fix $\gamma > 0$. Denote the generalization gap as $\gen(f\circ \phi) = R_\mu(f \circ \phi) - \hat{R}_{\gamma, m}(f \circ \phi)$.
    The following bound holds for all $f\in \Psi$ with probability at least $1-\delta > 0$:
    \begin{align*}
        \gen(f\circ \phi) \leq &\mathbb{E}_{c \sim \pi}\left[\frac{\operatorname{Lip}\left(\rho_f(\cdot, c)\right)}{\gamma} \operatorname{Var}_{m_c}\left(\phi_{\#} \mu_c\right)\right]\\
        & +\sqrt{\frac{\log (1 / \delta)}{2 m}},
    \end{align*}
    where 
    $\Lip\left(\rho_f(\cdot, c)\right)=\sup _{x, x^{\prime} \in \mathcal{X}} \frac{\left|\rho_f(\phi(x), c)-\rho_f\left(\phi\left(x^{\prime}\right), c\right)\right|}{\left\|\phi(x)-\phi\left(x^{\prime}\right)\right\|_2}$ is the margin Lipschitz constant w.r.t $\phi$.
\end{restatable}

\subsection{Total Variation and KL-divergence}
\begin{definition}[Total Variation]
    The \emph{total variation} between two distributions $\mu$ and $\nu$ on $\mathbb{R}^d$ is
    \begin{equation*}
    \operatorname{TV}(\mu, \nu):= \sup_{A\in \mathbb{R}^d}\{\mu(A) - \nu(A)\}.
    \end{equation*}
\end{definition}
where $\mu(A)$ and $\nu(A)$ are the cumulative probabilities on $\mu$ and $\nu$ respectively over the set $A$.

\begin{restatable}[Theorem 6.15~\cite{villani2009optimal}]{theorem}{villanitvbound}
\label{thm:wassertein_tv}
    The 1-Wasserstein distance is dominated by the total variation, i.e. let $\mu$ and $\nu$ be two distributions on $\mathcal{M}\in \mathbb{R}^d$, $\Delta(\mathcal{M})$ be the diameter of $\mathcal{M}$,
    we have $\mathcal{W}_1(\mu,\nu)\leq \Delta(\mathcal{M})\operatorname{TV}(\mu,\nu)$.
\end{restatable}

\begin{restatable}[Pinsker’s and Bretagnolle–Huber’s (BH) inequalities~\cite{polyanskiy2014lecture}]{lemma}{BHinequality}
\label{lemma:tv_dkl}
Let $\mu$ and $\nu$ be two probability distributions on $\mathcal{M}$, and $D_\text{KL}(\mu\|\nu)$ be the KL-divergence. Then
\begin{equation*}
    \operatorname{TV}(\mu,\nu) \leq \sqrt{\min(\frac{1}{2}D_\text{KL}(\mu\|\nu), 1-\exp(-D_\text{KL}(\mu\|\nu)))}.
\end{equation*}
\end{restatable}

\subsection{$\mathcal{F}$-pattern Trees and Homomorphism}
We introduce Theorem 1 from~\citet{Barcelo2021-rs}.

\begin{restatable}[Theorem 1 of~\citet{Barcelo2021-rs}]{theorem}{barcelotheorem1}
\label{lemma:barcelo1}
For any finite collection $\mathcal{F}$ of patterns, vertices $v$ and $w$ in a graph $G$ $\hom(T^r,G^v) = \hom(T^r,G^w)$ for every rooted $\mathcal{F}$-pattern tree $T^r$. Similarly, $G$ and $H$ are undistinguishable by the $\mathcal{F}$-WL-test if and only if $\hom(T,G) = hom(T,H)$, for every (unrooted) $\mathcal{F}$-pattern tree.
\end{restatable}
\cref{lemma:barcelo1} is generalized to $\mathcal{F}$-WL of $L$ iterations by~\citet{Barcelo2021-rs} in their proof. 

\begin{restatable}[Proof of Theorem 1 in~\citet{Barcelo2021-rs}]{theorem}{extendedbarcelotheorem1}
\label{lemma:barceloe1extended}
For any finite collection $\mathcal{F}$ of patterns, graphs $G$ and $H$, vertices $v\in V_G$ and $w\in V_H$ and $L\geq 0$:
\begin{equation*}
    (G,v)\equiv^{(L)}_{\mathcal{F}\text{-WL}}(H,w) \Longleftrightarrow \hom(T^r, G^v) = \hom(T^r, H^w),
\end{equation*}
for every $\mathcal{F}$-pattern tree $T^r$ of depth at most $L$. Similarly,
\begin{equation*}
    G\equiv^{(L)}_{\mathcal{F}\text{-WL}}H \Longleftrightarrow \hom(T, G) = \hom(T, H),
\end{equation*}
for every (unrooted) $\mathcal{F}$-pattern tree of depth at most $L$.
\end{restatable}

Next we show a connection between the homomorphism vector and $\Fc$-MPNN graph representation. The connection is used in the later proofs.

Let $\Gc$  be an input graph space, $\Nb^d$ be the space of homomorphism vector of the dimension $d$, and $\Nb^{d'}$ be the space of $\Fc$-MPNN graph embeddings of the dimension $d'$. And $\kappa:\Gc\to \Nb^d$ and 
$\phi:\Gc\to \Rb^{d'}$ two functions parameterised by $\Fc$ and $L$. 
Suppose that for any $x$ and $x'$ in $\Gc$, we have
$$
\kappa(x)=\kappa(x') \Longrightarrow \phi(x)=\phi(x').
$$
We then say that $\kappa $ \emph{bounds $\phi$ in distinguishing power} and write $\kappa\sqsubseteq \phi$.
We have
\begin{lemma}
\label{lem:f_phi_kapp}
Let $\kappa:\Gc\to \Nb^d$ and 
$\phi:\Gc\to \Rb^{d'}$ be such that $\kappa\sqsubseteq \phi$. Then, there exists an $f:\Nb^d\to\Rb^{d'}$ such that 
$\phi=f\circ \kappa$.
\end{lemma}
\begin{proof}
We define the function
$f:\Nb^d\to\Rb^{d'}$, as follows. Let $z\in \Nb^{d}$ and $x\in\Gc$ such that $\kappa(x)=z$. Then,
define $f(z):=\phi(x)\in\Rb^{d'}$.
Observe first that $f$ is well-defined. Indeed, if we take another $x'\in\Gc$ such that $\kappa(x')=z$, then $\kappa(x)=\kappa(x')$ and hence also $\phi(x')=\phi(x)=f(z)$ since $\kappa\sqsubseteq \phi$.
Also,
$
f\bigl(\kappa(x)\bigr)=\phi(x)
$, by definition.
\end{proof}

\begin{lemma}
\label{lem:push_forwward_equal}
Let $\kappa:\Gc\to \Nb^d$ and 
$\phi:\Gc\to \Rb^{d'}$ be such that $\kappa\sqsubseteq \phi$. Let $\mu$ be a probability distribution on $\Gc$. 
Consider a function $f:\Nb^d\to\Rb^{d'}$ such that 
$\phi=f\circ \kappa$.
Then, we have the following equality between pushforward distributions on $\Rb^{d'}$
$$
f_\#\bigr(\kappa_\#(\mu)\bigr)=\phi_\#(\mu).
$$
\end{lemma}
\begin{proof}
This is just by definition. Indeed,
on the hand, for $I\subseteq\Rb^{d'}$
$$\phi_\#(\mu)(I):=\mu\bigl(
\{x\in\Gc\mid \phi(x)\in I\}
\bigr).$$
On the other hand,
\begin{align*}
f_\#\bigl(\kappa_\#(\mu)\bigr)(I)&=\kappa_\#(\mu)\bigl(
\{z\in\Rb^{d'}\mid f(z)\in I\}
\bigr)\\
&=\mu\bigl(x\in\Gc\mid f(\kappa(x))\in I
\bigr).
\end{align*}
By \cref{lem:f_phi_kapp}, $f\circ\kappa=\phi$, from which the identity $f_\#\bigl(\varphi_\#(\mu)\bigr)=\kappa_\#(\mu)$
follows.
\end{proof}
Finally, we have
\begin{corollary}
\label{cor:KL_pushforward}
Let $\kappa:\Gc\to \Nb^d$ and 
$\phi:\Gc\to \Rb^{d'}$ be such that $\kappa\sqsubseteq \phi$. Then, for
distributions $\mu$ and $\nu$ on $\Gc$,
$$
D_\text{KL}\bigl(\phi_\#(\mu)\|\phi_\#(\nu)\bigr)\leq 
D_{\text{KL}}\bigl(\kappa_\#(\mu)\|\kappa_\#(\nu)\bigr).
$$
\end{corollary}
\begin{proof}
From \cref{lem:push_forwward_equal} we have
$\phi_\#(\mu)=f\bigl(\kappa_\#(\mu)\bigr)$ and 
$\phi_\#(\nu)=f\bigl(\kappa_\#(\nu)\bigr)$, then this is an easy consequences of "data processing inequality" for KL divergence~\citep{DBLP:books/daglib/cover2006} and \cref{lem:push_forwward_equal}.
\end{proof}

Now we prove the proposed bounds.

\dataindependentgraphbound*
\begin{proof}
Recall the bound in \cref{lemma:chuang_bound}, 
\begin{equation*}
    \gen(f\circ \phi) \leq \mathbb{E}_{c \sim p}\left[\frac{L_c}{\gamma} \operatorname{Var}_{m_c}\left(\phi^{\mathcal{F},L}_{\#} \mu_c\right)\right] 
    +\sqrt{\frac{\log (1 / \delta)}{2 m}}.
\end{equation*}
Now we seek to bound $\operatorname{Var_{m_c}}(\phi_{\#}^{\mathcal{F},L}\mu_{c})$.
By definition,  
\begin{equation*}
\operatorname{Var}_{m_c}(\phi_{\#}^{\mathcal{F},L}\mu_{c}) = \mathbb{E}_{S,\tilde{S}\sim \mu_c^{m_c}} [\mathcal{W}_1(\phi_{\#}^{\mathcal{F},L}\mu_S, \phi_{\#}^{\mathcal{F},L}\mu_{\tilde{S}})].
\end{equation*}
By \cref{thm:wassertein_tv}, we have
\begin{equation*}
\mathcal{W}_1(\phi_{\#}^{\mathcal{F},L}\mu_S, \phi_{\#}^{\mathcal{F},L}\mu_{\tilde{S}}) \leq \beta_c \operatorname{TV}(\phi_{\#}^{\mathcal{F},L}\mu_S, \phi_{\#}^{\mathcal{F},L}\mu_{\tilde{S}}).
\end{equation*}
Further by \cref{lemma:tv_dkl}, we have
\begin{equation*}
    \operatorname{TV}(\phi_{\#}^{\mathcal{F},L}\mu_S, \phi_{\#}^{\mathcal{F},L}\mu_{\tilde{S}}) \leq \Omega\left(D_{KL}\left(
        \phi_{\#}^{\mathcal{F},L}\mu_S\parallel \phi_{\#}^{\mathcal{F},L}\mu_{\tilde{S}}
    \right)\right).
\end{equation*}
By \cref{cor:KL_pushforward} we have
\begin{align*}
D_{KL}&\left(
     \phi_{\#}^{\mathcal{F},L}\mu_S, \phi_{\#}^{\mathcal{F},L}\mu_{\tilde{S}}
    \right) \\
    &\leq  D_{KL}\left( X_{\mu_S, T_L(\Fc)}\parallel  X_{\mu_{\tilde{S}}, T_L(\Fc)} \right),
\end{align*}
because the homomorphism vector $\Hom(T_L(\Fc), G)$ bounds $\phi^{\Fc,L}(G)$ for $G\in \Gc$, i.e.
$\Hom(T_L(\Fc), \cdot)\sqsubseteq \phi^{\Fc,L}$, and $\Hom_\#(T_L(\Fc), \mu_S) = X_{\mu_S, T_L(\Fc)}$.

The proof is done.

\end{proof}

\datadependentgraphbound*
\begin{proof}
Recall the bound in \cref{lemma:chuang_bound}, 
\begin{equation*}
    \gen(f\circ \phi) \leq \mathbb{E}_{c \sim p}\left[\frac{L_c}{\gamma} \operatorname{Var}_{m_c}\left(\phi^{\mathcal{F},L}_{\#} \mu_c\right)\right] 
    +\sqrt{\frac{\log (1 / \delta)}{2 m}}
\end{equation*}
\citet{Chuang2021-ik} (Lemma 5) show the $k$-variance $\operatorname{Var}_k(\mu)$ can be estimated empirically by $\widehat{\operatorname{Var}}_{k, n}(\phi_{\#}^{\mathcal{F},L}\mu_c)$,
\begin{equation*}
    \operatorname{Var}_{k}(\phi_{\#}^{\mathcal{F},L}\mu_c) \leq \widehat{\operatorname{Var}}_{k, n}(\phi_{\#}^{\mathcal{F},L}\mu_c) + \sqrt{\frac{2\beta_c^2\log(1/\delta)}{nk}}
\end{equation*}
where
\begin{equation*}
\widehat{\operatorname{Var}}_{k,n}(\phi_{\#}^{\mathcal{F},L}\mu_c)=\frac{1}{n}\sum_{j=1}^n\mathcal{W}_1(\phi_{\#}^{\mathcal{F},L}\mu_{S^j},\phi_{\#}^{\mathcal{F},L}\mu_{\tilde{S}^j})
\end{equation*}
$\widehat{\operatorname{Var}}_{k,n}(\phi_{\#}^{\mathcal{F},L}\mu_c)$ can be computed
using $n$ samples $\{S^j, \tilde{S}^j\}^n_{j=1}$ where each $S^j, \tilde{S}^j \sim \mu^{k}$. 
With probability at least $1-\delta$, for $m=\sum_{c=1}^K\lfloor\frac{m_c}{2n}\rfloor$, $\gen(f\circ \phi^{\mathcal{F},L})$ is bounded by
\begin{align}
    \label{eqn:data_graph_bound_intermediate}
    \mathbb{E}_{c \sim \pi}
    & \left[
    \frac{L_c}{\gamma} 
    \left(
        \widehat{\operatorname{Var}}_{\lfloor\frac{m_c}{2n}\rfloor, n}(\phi_{\#}^{\mathcal{F},L}\mu_c) + 2\beta_c\sqrt{\frac{\log(\frac{2K}{\delta})}{n\lfloor\frac{m_c}{2n}\rfloor}}
    \right)
    \right] \nonumber \\
    & +\sqrt{\frac{\log (2 / \delta)}{2 m}}
\end{align}
Now we seek to upper bound the term $\mathcal{W}_1(\phi_{\#}^{\mathcal{F},L}\mu_{S^j}, \phi_{\#}^{\mathcal{F},L}\mu_{\tilde{S}^{j}})$ in $\widehat{\operatorname{Var}}_{k,n}(\phi_{\#}^{\mathcal{F},L}\mu_c)$ using KL divergence.

By \cref{thm:wassertein_tv}, we have
\begin{equation*}
\mathcal{W}_1(\phi_{\#}^{\mathcal{F},L}\mu_{S^j}, \phi_{\#}^{\mathcal{F},L}\mu_{\tilde{S}^j}) \leq \beta_c \operatorname{TV}(\phi_{\#}^{\mathcal{F},L}\mu_{S^j}, \phi_{\#}^{\mathcal{F},L}\mu_{\tilde{S}^j})
\end{equation*}
Further by \cref{lemma:tv_dkl}, we have
\begin{equation*}
    \operatorname{TV}(\phi_{\#}^{\mathcal{F},L}\mu_{S^j}, \phi_{\#}^{\mathcal{F},L}\mu_{\tilde{S}^j}) \leq \Omega\left(D_{KL}\left(
        \phi_{\#}^{\mathcal{F},L}\mu_{S^j}, \phi_{\#}^{\mathcal{F},L}\mu_{\tilde{S}^j}
    \right)\right)
\end{equation*}
By \cref{cor:KL_pushforward} we have
\begin{align*}
D_{KL}&\left(
     \phi_{\#}^{\mathcal{F},L}\mu_{S^j}, \phi_{\#}^{\mathcal{F},L}\mu_{\tilde{S}^j}
    \right) \\
    &\leq D_{KL}\left( X_{\mu_{S^j}, T_L(\Fc)}\parallel  X_{\mu_{\tilde{S}^j}, T_L(\Fc)} \right),
\end{align*}
because the homomorphism vector $\Hom(T_L(\Fc), G)$ bounds $\phi^{\Fc,L}(G)$ for $G\in \Gc$, i.e.
$\Hom(T_L(\Fc), \cdot)\sqsubseteq \phi^{\Fc,L}$, and $\Hom_\#(T_L(\Fc), \mu_S^j) = X_{\mu_S^j, T_L(\Fc)}$.

Putting together, we have
\begin{align*}
    \widehat{\operatorname{Var}}&_{k,n}(\phi_{\#}^{\mathcal{F},L}\mu_c) \\
    & \leq \frac{1}{n}\sum_{j=1}^n
    \left(\beta_c \Omega\left(D_{KL}\bigl( X_{\mu_{S^j}, T_L(\Fc)}\parallel  X_{\mu_{\tilde{S}^j}, T_L(\Fc)} \bigr)\right) \right)
\end{align*}
The proof is done.
\end{proof}

\dataindependentnodebound*
\begin{proof}
The proof is carried out in a similar way as \cref{lemma:non_class_bound}.
\end{proof}

\datadependentnodebound*
\begin{proof}
The proof is carried out in a similar way as \cref{lemma:graph_class_bound}.
\end{proof}

\dataindependentcomparison*
\begin{proof}
We prove the bullet points one by one.
\begin{enumerate}
    \item We look at the two controlling factor of the bound separately: $\widetilde{D}_{KL}(\Fc, S, \tilde{S})$ and $\beta_c$.
    We know that $|T_L(\mathcal{F})|$ grows with $L$. And we know larger $|T_L(\mathcal{F})|$ will result in increased or equivalent $\widetilde{D}_{KL}(\Fc, S, \tilde{S})$. Now we look at $\beta_c$, given $\beta_c$ captures the diameter of $\phi^{\mathcal{F},L}_{\#}\mu_c$, $\beta_c$ at $L+1$ is larger or equal to $\beta_c$ at $L$. So the bound described in \cref{lemma:non_class_bound}, at layer $L+1$, is larger or equal to the bound at layer $L$.
    \item This bullet point can be proved in a similar way as the first bullet point.
    \item It is easy to see that when $\mathcal{F}_{\hom} = \{\FlagGraph{1}{}\}$, $\Fc$-MPNN is equivalent to 1-WL GNN, since $T_L(\{\FlagGraph{1}\})$ contains all trees up to depth $L$~\cite{Barcelo2021-rs}. So when $|\mathcal{F}|$ increases, $|T_L(\mathcal{F})|$ increases too. As a result, the bound value at larger $|\Fc|$ will be equivalent or larger than the one at smaller $\Fc$.
    \item From \cite{Barcelo2021-rs}, we know that $\mathcal{F}$ contains infinite number of graphs of treewidth bounded by k when expressed k-WL under $\Fc$-WL. Hence, it is easy to see that for HI-GNN and SI-GNN of finite patterns of treewidth bounded by k, the correspoding pattern set is a subset of the one for k-WL. So by the second bullet point we can land on the conclusion.
\end{enumerate}
\end{proof}






\end{document}